\documentclass[a4paper, 11pt]{article}

\usepackage{subfigure}
\usepackage{amsthm}
\usepackage{amsmath, amssymb, amsfonts}
\usepackage{amssymb, epsfig}
\usepackage{authblk}
\usepackage[authoryear]{natbib}
\usepackage{mathrsfs}
\usepackage{enumitem} 
\usepackage{microtype}

\usepackage[backref, colorlinks=true]{hyperref}
\hypersetup{
	linktocpage={true},
	linkcolor = {red},
	citecolor = {blue} 
}
\usepackage[ruled, vlined, linesnumbered]{algorithm2e}
\usepackage[font=footnotesize, labelfont=bf]{caption}

\usepackage{bm}
\usepackage{tikz}
\usepackage{graphicx}
\usetikzlibrary{positioning}

\usepackage{parskip}

\usepackage{titlesec}
\titleformat{\section}{\large\bfseries}{\thesection}{1em}{}
\titleformat{\subsection}{\normalsize\bfseries}{\thesubsection}{1em}{}

\pdfpageheight\paperheight
\pdfpagewidth\paperwidth

\def\E{\ifmmode{\mathbb E}\else{$\mathbb E$}\fi} 
\def\N{\ifmmode{\mathbb N}\else{$\mathbb N$}\fi} 
\def\R{\ifmmode{\mathbb R}\else{$\mathbb R$}\fi} 
\def\Q{\ifmmode{\mathbb Q}\else{$\mathbb Q$}\fi} 
\def\C{\ifmmode{\mathbb C}\else{$\mathbb C$}\fi} 
\def\H{\ifmmode{\mathbb H}\else{$\mathbb H$}\fi} 
\def\Z{\ifmmode{\mathbb Z}\else{$\mathbb Z$}\fi} 
\def\P{\ifmmode{\mathbb P}\else{$\mathbb P$}\fi} 
\def\T{\ifmmode{\mathbb T}\else{$\mathbb T$}\fi} 
\def\SS{\ifmmode{\mathbb S}\else{$\mathbb S$}\fi} 
\def\DD{\ifmmode{\mathbb D}\else{$\mathbb D$}\fi} 

\newcommand{\bse}{\begin{subequations}}
\newcommand{\ese}{\end{subequations}}
\newcommand{\ben}{\begin{enumerate}}
\newcommand{\een}{\end{enumerate}}
\newcommand{\bens}{\begin{enumerate*}}
\newcommand{\eens}{\end{enumerate*}}
\newcommand{\be}{\begin{equation}}
\newcommand{\ee}{\end{equation}}
\newcommand{\bea}{\begin{eqnarray}}
\newcommand{\eea}{\end{eqnarray}}
\newcommand{\baa}{\begin{eqnarray*}}
\newcommand{\eaa}{\end{eqnarray*}}
\newcommand{\bc}{\begin{center}}
\newcommand{\ec}{\end{center}}

\newtheorem{theorem}{Theorem}

\theoremstyle{corollary}

\theoremstyle{lemma}
\newtheorem{lemma}{Lemma}

\theoremstyle{proposition}
\newtheorem{proposition}{Proposition}

\theoremstyle{axiom}

\theoremstyle{conjecture}

\theoremstyle{example}
\newtheorem{example}{Example}

\theoremstyle{definition}
\newtheorem{definition}{Definition}

\theoremstyle{remark}

\graphicspath{{./}{../figures/}}

\begin{document}

\title{\Large\bfseries Differentially private online Bayesian estimation\\ with adaptive truncation}


\author[1, 2]{Sinan Yıldırım}

\affil[1]{\normalsize Faculty of Engineering and Natural Sciences, Sabancı University, İstanbul, Turkey}
\affil[2]{\normalsize Center of Excellence in Data Analytics (VERİM), Sabancı University, İstanbul, Turkey}
\affil[1]{\small{\texttt{sinanyildirim@sabanciuniv.edu}}}


\maketitle

\begin{abstract}
We propose a novel online and adaptive truncation method for differentially private Bayesian online estimation of a static parameter regarding a population. We assume that sensitive information from individuals is collected sequentially and the inferential aim is to estimate, on-the-fly, a static parameter regarding the population to which those individuals belong. We propose sequential Monte Carlo to perform online Bayesian estimation. When individuals provide sensitive information in response to a query, it is necessary to perturb it with privacy-preserving noise to ensure the privacy of those individuals. The amount of perturbation is proportional to the sensitivity of the query, which is determined usually by the range of the queried information. The truncation technique we propose adapts to the previously collected observations to adjust the query range for the next individual. The idea is that, based on previous observations, we can carefully arrange the interval into which the next individual's information is to be truncated before being perturbed with privacy-preserving noise. In this way, we aim to design predictive queries with small sensitivity, hence small privacy-preserving noise, enabling more accurate estimation while maintaining the same level of privacy. To decide on the location and the width of the interval, we use an exploration-exploitation approach \emph{a la} Thompson sampling with an objective function based on the Fisher information of the generated observation. We show the merits of our methodology with numerical examples.

\end{abstract}

\section{Introduction} \label{sec: Introduction}
During the past couple of decades, there has been a rapid increase in the amount of collected data as well as concerns about individuals' privacy. This has made privacy-preserving data analysis a popular and important subject in data science. Along the way, \emph{differential privacy} has become a popular framework for privacy-preserving data sharing algorithms \citep{Dwork_2006, Dwork_and_Roth_2013}.

There are two conflicting interests in privacy-preserving data analysis: (i) The individuals of a population who contribute to a data set with their sensitive information want to protect their privacy against all possible adversaries. (ii) Conflicting with that, it is desired to be able to estimate a common quantity of interest regarding the population based on sensitive data with reasonable accuracy. To put the conflict in a statistical context, we let $X_{t} \sim \mathcal{P}_{\theta}$ be the sensitive information of $t$'th individual of a sample randomly chosen from a large population with a population distribution $\mathcal{P}_{\theta}$. We want to estimate $\theta$ while also protecting the privacy of the individuals contributing to the sample, i.e., without revealing `much'  information about $X_{t}$s individually. 

In this paper, we are particularly interested in online Bayesian estimation of $\theta$ as we continually collect $Y_{1}, Y_{2}, \ldots$, which are the \emph{perturbed} versions of $X_{1}, X_{2}, \ldots$ respectively. The cases where individuals contribute to a data set continually are not rare: Imagine web users registering to a web application after entering their information, patients being admitted to a hospital, customers applying for a bank loan, etc. We address two interrelated questions:
\begin{itemize}
\item How can we improve the estimate of $\theta$ as we collect $Y_{1}, Y_{2}, \ldots$ continually?
\item As we estimate $\theta$, how can we continually adjust the privacy-preserving mechanism that generates $Y_{t}$ from $X_{t}$ so that the estimation performance is improved as $t$ increases?
\end{itemize}
 
Differentially private Bayesian inference of $\theta$ has been the subject of several recent studies, with Monte Carlo being the main methodological tool for inference. Stochastic gradient MCMC algorithms were proposed in \citet{wang_et_al_2015, Li_et_al_2019}, while reversible MCMC algorithms were proposed in \citet{Heikkila_et_al_2019, Yildirim_and_Ermis_2019, Raisa_et_al_2021}. Those algorithms require as many interactions with sensitive data as the number of iterations they run for. An alternative scheme to that is called input perturbation, where the sensitive data are perturbed and shared once and for all, and all the subsequent Bayesian inference is performed on the perturbed data without further interaction with the sensitive data \citep{foulds_et_al_2016, Williams_and_McSherry_2010, Karwa_et_al_2014, Bernstein_and_Sheldon_2018,  Park_et_al_2021, Gong_2022, Alparslan_and_Yildirim_2022, Ju_et_al_2022}.

All the cited works above consider differentially private Bayesian inference conditional on a batch (static) data set. Unlike those works, in this paper, we consider the case with continual observations, where data from the individuals are collected \emph{sequentially} in a privacy-preserving way. This scenario enables two methodological opportunities and/or challenges: 
\begin{enumerate}
\item One can (and/or should) estimate the static parameter on-the-fly, that is, update the estimate as data are being received. Differentially private estimation under continual observation has been the subject of several works that are initiated by \citet{Dwork_et_al_2010b}; other important contributions include \citet{Chan_et_al_2010, Cao_et_al_2017}. However, those works are usually applied to online tracking of dynamic summaries of data, such as the count of a certain property, rather than estimating a static parameter of the population from which the sensitive data are being received. In particular, they do not consider Bayesian estimation.
\item As we estimate the parameter, we can adaptively adjust the query for the next individual's information to make the response as \emph{informative} as possible. For example, if, based on the noisy income values collected so far from 100 individuals, we have estimated that the mean income of the population is around $\hat{\mu}$, we can ask the $101$'th individual to provide their income information after \emph{truncating} it to an interval around $\hat{\mu}$, such as $[\hat{\mu} - \Delta, \hat{\mu} + \Delta]$, and \emph{then} privatising it by adding noise to the (possibly) truncated value.

The motivation behind pursuing such an adaptive truncation technique is to improve the estimation performance with less noisy data while maintaining a given level of privacy. The standard deviation of the privacy-preserving noise added to the outcome of a query is proportional to the sensitivity of the query. By default, the queried information may be unbounded or have very large ranges, resulting in low utility. Continuing with the income example above, assume that the natural limits of an income are $[x_{\min}, X_{\max}]$ so that a query that directly asks for income information has a sensitivity of $X_{\max} - x_{\min}$, which is expectedly large. If adaptive truncation were used, instead, referring to the above example, the query interval for $101$'th individual would be $[\hat{\mu} - \Delta, \hat{\mu} + \Delta]$ with sensitivity $2 \Delta$.
\end{enumerate}

Truncation is considered in many works as a natural way to have finite sensitivity, see \citet{Heikkila_et_al_2017, Ju_et_al_2022} for examples of differentially private Bayesian estimation based on truncated data. Those works regard estimation based on batch data; adaptive truncation during online Bayesian learning, as done in this paper, is not considered.

This paper contributes to the literature on differential privacy by addressing the two challenges described above with a novel methodology. For the first challenge, that is, online Bayesian estimation of $\theta$, we propose a sequential Monte Carlo (SMC) method for static parameter estimation as studied in \citet{gilks_and_berzuini_2001, Chopin_2002}. For the second challenge, we propose a novel adaptive truncation method that employs an \emph{exploration-exploitation} heuristic to maximise the aggregate `information' in the sequence of observations $Y_{1}, Y_{2}, \ldots$ about $\theta$. To measure the amount of `information', we choose the Fisher information as suggested in \citet{Alparslan_and_Yildirim_2022}. As we show in Section \ref{sec: Adaptive truncation for the transformation}, the \emph{exploration} part of the proposed approach can be seen as an instance of Thompson sampling \citep{Russo_et_al_2018} from reinforcement learning. The \emph{exploitation} part consists of finding the truncation points that make the resulting observations most informative in terms of Fisher information. Finally, for the \emph{exploitation} step, we pay special attention to \emph{location-scale} families and show that the maximisation task can be performed for all time steps once and for all. To the best of our knowledge, this is the first work that tackles the problem of online differentially private Bayesian estimation with adaptive queries.

The paper is organised as follows. In Section \ref{sec: Differential Privacy}, we introduce the basic concepts of differential privacy. In Section \ref{sec: Adaptive differentially private parameter estimation}, we discuss the problem of online parameter estimation using privatised noisy statistics of the sensitive data and present our methodology in general. In Sections \ref{sec: Sequential Monte Carlo for Bayesian estimation} and \ref{sec: Adaptive truncation for the transformation}, we describe the details of our methodology. In Section \ref{sec: Numerical results} we present the results of some numerical experiments. Finally, we give our concluding remarks in Section \ref{sec: Conclusion}. This paper has an Appendix section for some deferred details.

\section{Differential Privacy} \label{sec: Differential Privacy}

Let $\mathcal{X}$ be a set of individual data values and $\mathscr{X} = \bigcup_{n = 1}^{\infty} \mathcal{X}^{n}$ be the set of all possible data sets.  Define the Hamming distance between the data sets $\bm{x}, \bm{x}' \in \mathscr{X}$ as the number of different elements between those data sets, denoted by $h(\bm{x}, \bm{x}')$. We call two data sets $\bm{x}, \bm{x}' \in \mathscr{X}$ neighbours if $h(\bm{x}, \bm{x}') = 1$. A \emph{randomised algorithm} can be defined as a couple $\mathcal{A} = (A, \mu)$, where $A: \mathscr{X} \times \mathcal{E} \mapsto \mathcal{Y}$ is a function and $\mu$ is a probability distribution on $\mathcal{E}$, which represents the randomness intrinsic to $\mathcal{A}$. Upon taking an input $\bm{x} \in \mathscr{X}$, the randomised algorithm $\mathcal{A}$ generates random numbers $\omega \sim \mu(\cdot)$ in  $\mathcal{E}$ and outputs $A(\bm{x}, \omega)$. Differential privacy \citep{Dwork_2006} quantifies a certain sense of similarity between random outputs $A(\bm{x}, \omega)$ and $A(\bm{x}', \omega)$ when $\bm{x}$ and $\bm{x}'$ are neighbours.
\begin{definition}[Differential privacy (DP)] \label{defn: Differential privacy}
A randomised algorithm $\mathcal{A} = (A, \mu)$ is $(\epsilon, \delta)$-DP if 
\[
\mathbb{P} \left[ A(\bm{x}, \omega)\in S \right]  \leq e^{\epsilon} \mathbb{P} \left[ \mathcal{A}(\bm{x}', \omega)\in S \right] + \delta, \quad \forall \bm{x}, \bm{x}' \in \mathscr{X} s.t.\ h(\bm{x}, \bm{x}') = 1, \quad \forall S \subseteq \mathcal{Y},
\] 
where the randomness is with respect to $\omega \sim \mu(\cdot)$. We say $\mathcal{A}$ is $\epsilon$-DP when $\delta = 0$.
\end{definition}
As far as privacy is concerned, both privacy parameters $(\epsilon, \delta)$ are desired to be as small as possible. The following theorem states that $(\epsilon, \delta)$-DP is maintained by post-processing the output of an $(\epsilon, \delta)$-DP algorithm.
\begin{theorem}[Post-processing] \label{thm: post-processing}
Define functions $A_{1}: \mathscr{X} \times \mathcal{E}_{1} \mapsto \mathcal{Y}_{1}$ and $A_{2}: \mathcal{Y}_{1} \times \mathcal{E}_{2} \mapsto \mathcal{Y}_{2}$; and probability distributions $\mu_{1}$, $\mu_{2}$ on $\mathcal{E}_{1}$, $\mathcal{E}_{2}$, respectively. Furthermore, let $A: \mathscr{X} \times \mathcal{E}_{1} \times \mathcal{E}_{2} \mapsto \mathcal{Y}_{2}$ be defined by $A(\bm{x}, \omega_{1}, \omega_{2}) = A_{2} (A_{1}(\bm{x}, \omega_{1}), \omega_{2})$, and $\mu = \mu_{1} \otimes \mu_{2}$. Then, if $\mathcal{A}_{1} = (A_{1}, \mu_{1})$ is $(\epsilon, \delta)$-DP, $\mathcal{A} = (A, \mu)$ is $(\epsilon, \delta)$-DP, too.
\end{theorem}
Let $\varphi: \mathscr{X} \mapsto \mathbb{R}$ be a function and assume that $\varphi(\bm{x})$ is queried. One common way of achieving differential privacy, in this case, is the \emph{Laplace mechanism} \citep{Dwork_2008}, which relies on the \emph{$L_{1}$-sensitivity} of $\varphi$, given by
\begin{equation} \label{eq: L1 sensitivity}
\Delta \varphi = \sup_{\substack{\bm{x}, \bm{x}' \in \mathscr{X}: \\ h(\bm{x}, \bm{x}') = 1}} \vert \varphi(\bm{x}) - \varphi(\bm{x}') \vert.
\end{equation}

\begin{theorem}[Laplace mechanism] \label{thm: Laplace mechanism}
The algorithm that returns $\varphi(\bm{x}) + \Delta \varphi  V$ given the input $\bm{x} \in \mathscr{X}$, where $V \sim \textup{Laplace}(1 /\epsilon)$, is $\epsilon$-DP.
\end{theorem}
Other useful definitions of data privacy have close relations to differential privacy. Some important examples are Gaussian differential privacy \citep{Dong_et_al_2022} and zero-concentrated differential privacy \citep{Bun_and_Steinke_2016}, both of which promote the \emph{Gaussian mechanism} \citep{Dwork_and_Roth_2013} (where $V$ in Theorem \ref{thm: Laplace mechanism} has a normal distribution) as its primary mechanism for providing privacy. The Gaussian mechanism can also provide $(\epsilon, \delta)$-DP for $\delta > 0$ if the variance is modified to depend on $\delta$ also.

For the rest of the paper, we will consider the Laplace mechanism to provide $\epsilon$-DP for sake of simplicity. We remark, however, that other additive mechanisms to provide privacy in other senses also fit into our methodology with minor changes. In particular, our methodology applies to the Gaussian mechanism in an almost identical manner. 

\section{Differentially private parameter estimation with adaptive queries} \label{sec: Adaptive differentially private parameter estimation}

Assume a sequence of i.i.d.\ data points
\[
X_{t} \overset{\textup{i.i.d.}}{\sim} \mathcal{P}_{\theta}, \quad t \geq 1,
\]
where $X_{t}$ is some sensitive information that belongs to the $t$'th individual sampled from a population. We want to estimate the unknown parameter $\theta$ of the population distribution $\mathcal{P}_{\theta}$. However, we are not allowed to access to $X_{t}$'s directly; instead, individuals share their information through a function $s_{t}: \mathcal{X} \mapsto \mathbb{R}$ and with privacy-preserving noise as
\begin{equation} \label{eq: noisy observations}
Y_{t} = s_{t}(X_{t}) + \Delta s_{t}  V_{t}, \quad V_{t} \overset{\textup{i.i.d.}}{\sim} \textup{Laplace}(1/\epsilon), \quad t \geq 1,
\end{equation}
where $\Delta s_{t}$ is the sensitivity defined as in \eqref{eq: L1 sensitivity}, that is,
\[
\Delta s_{t} =  \sup_{x, x' \in \mathcal{X}} \vert s_{t}(x) - s(x') \vert.
\]
 We consider online estimation of $\theta$ when $\{Y_{t} \}_{t \geq 1}$ are observed sequentially in time. The recursion that corresponds to a sequential estimation procedure can be written down generically as
\[
\Theta_{t} = G(\Theta_{t-1}, Y_{1:t}, s_{1:t}).
\]
The update function $G$ produces $\Theta_{t}$ using all the information up to time $t$, which includes $\Theta_{t-1}$, the functions $s_{1:t}$, and the observations $Y_{1:t}$. Generally, $\Theta_{t}$ is not necessarily a point estimate but a collection of variables needed to construct the estimation of $\theta$ at time $t$. For example, in SMC for Bayesian estimation, $\Theta_{t}$ can correspond to the particle system at time $t$. Details of such an algorithm will be provided in Section \ref{sec: Sequential Monte Carlo for Bayesian estimation}.

This paper focuses on the question of whether it is possible to choose $s_{t}$ \emph{adaptively} so that $\theta$ is estimated with improved accuracy relative to its non-adaptive counterpart. The choice of $s_{t}$ is important because $s_{t}$ determines how much information is contained in $Y_{t}$ about $\theta$ in two ways \citep{Alparslan_and_Yildirim_2022}: 
\begin{itemize}
\item The first way is related to the sufficiency or informativeness of $s_{t}$ in the classical sense. For example, let $\mathcal{P}_{\theta} = \mathcal{N}(\theta, 1)$ with an unknown mean $\theta$. Then, discarding the privacy-preserving noise, $s_{t}(x_{t}) = x_{t}$ would be a better choice than $s_{t}(x_{t}) = |x_{t}|$  since $|x_{t}|$ masks the information that is contained in $x_{t}$ about $\theta$. 
\item Secondly,  the standard deviation of the privacy-preserving noise is proportional to the sensitivity $\Delta s_{t}$. A mild truncation results in a large $\Delta s_{t}$, which necessitates too much privacy-preserving noise. (As an extreme case, think of an unbounded $s_{t}$). On the flip side, making $\Delta s_{t}$ too small could result in a small amount of information in $s_{t}(X_{t})$ about $\theta$. (Imagine a constant $s_{t}(\cdot)$, which has $\Delta s_{t} = 0$ but carries no information about $\theta$.) Therefore, truncation and sensitivity establish a trade-off. Below, we exemplify the trade-off when $s_{t}$ is a truncation function.
\end{itemize}

\begin{example}
Assume that our goal is to learn the average income $\theta$ of the individuals in a given population, with a population distribution $\mathcal{N}(\theta, \sigma^{2})$, where $\sigma^{2}$ is known. Assume that data is collected from (some of) the individuals in this population in a sequential way.  However, since the income information is sensitive, each individual's income is recorded (or shared by the individual) with privacy-preserving noise as in \eqref{eq: noisy observations}. Consider the specific choice
\[
Y_{t} = \min\{ \max\{ X_{t}, l \}, r \} + (r - l) V_{t}, \quad V_{t} \sim  \textup{Laplace}\left(1/ \epsilon \right).
\]
If the interval $[l, r]$ is wide, true income $X_{t}$ is not likely to be truncated but $Y_{t}$ suffers a large noise for ensuring the given level of privacy. On the other hand, if $[l, r]$ is small, $X_{t}$ is likely to be truncated but $Y_{t}$ is less noisy. This makes a trade-off between truncation and privacy-preserving noise, the two undesired components in terms of statistical inference. 

It would therefore be reasonable to carefully adjust the interval adaptively as we collect data, where the interval for receiving the $t$'th individual's data is denoted by $[l_{t}, r_{t}]$. Intuitively, we would aim to set $[l_{t}, r_{t}]$ in such a way that it will likely contain the true value and it is small so that the required privacy-preserving noise has a small variance. If $\theta$ is a location parameter, we could do that by positioning $[l_{t}, r_{t}]$ around the most recent estimate of $\theta$. 
\end{example}

The general online estimation method with adaptive functions $s_{t}$ is given in Algorithm \ref{alg: Differentially private online learning}. The algorithm outlines the general idea in this paper: We gain knowledge about $\theta$ as we observe $Y_{t}$'s; which we use to adapt the statistic $s_{t+1}$ such that the new observation $Y_{t+1}$ carries more information about $\theta$ than it would with an arbitrary choice of $s_{t+1}$. 

Algorithm \ref{alg: Differentially private online learning} is $\epsilon$-DP. Each observation $Y_{t}$ belongs to an individual and is shared with $\epsilon$-DP. Furthermore, all the updates in Algorithm \ref{alg: Differentially private online learning} are performed using the shared data $\{ Y_{t} \}_{t \geq 1}$ and \emph{not} the private data $\{ X_{t} \}_{t \geq 1}$. Therefore, by Theorem \ref{thm: post-processing}, those updates do not introduce any further privacy leaks. A more formal statement in Proposition \ref{alg: Differentially private online learning}, a proof can be found in Appendix \ref{sec: Proof of Proposition 1}.

\begin{algorithm}[t]
\caption{Differentially private online learning - general scheme}
\label{alg: Differentially private online learning}
Initialise the estimation system $\Theta_{0}$ and $s_{1}(\cdot)$.\\

\For{$t = 1, 2, \ldots$}{
The function $s_{t}$ is revealed to individual $t$, which shares his/her data $X_{t}$ as
\[
Y_{t} = s_{t}(X_{t}) + \Delta s_{t} V_{t}, \quad V_{t} \sim \textup{Laplace}\left( 1/ \epsilon \right).
\]

Update the estimation system $\Theta_{t}$ as
\begin{equation} \label{eq: update theta general}
\Theta_{t} = G(\Theta_{t-1}, Y_{1:t}, s_{1:t}).
\end{equation}

Update the new function
\begin{equation} \label{eq: update s general}
s_{t+1} = H(\Theta_{t}).
\end{equation}
}
\end{algorithm}

\begin{proposition} \label{prop: DP of Alg 1}
Algorithm \ref{alg: Differentially private online learning} is $\epsilon$-DP.
\end{proposition}
In Sections \ref{sec: Sequential Monte Carlo for Bayesian estimation} and \ref{sec: Adaptive truncation for the transformation}, we describe the methods for the updates in \eqref{eq: update theta general} and in \eqref{eq: update s general}.

\section{Sequential Monte Carlo for online Bayesian estimation} \label{sec: Sequential Monte Carlo for Bayesian estimation}
In this section, we focus on  $G$ in \eqref{eq: update theta general} in Algorithm \ref{alg: Differentially private online learning}, which stands for the parameter estimation update upon receiving a new observation. We consider the functions $s_{t}$ given and present an SMC method for online Bayesian estimation of $\theta$. SMC is a popular numerical method for online Bayesian inference; see \citet{gilks_and_berzuini_2001, Chopin_2002} for some pioneer works. Let $p_{\theta}(\cdot)$ be the probability density (or mass) function (pdf or pmf) of $\mathcal{P}_{\theta}$. With a prior distribution $\eta(\theta)$ on $\theta$, the following sequence of posterior distributions is targeted sequentially with SMC.
\begin{equation}
\begin{aligned} 
 p_{s_{1:t}}^{\epsilon}(\theta, x_{1:t} | y_{1:t}) &\propto p_{s_{1:t}}^{\epsilon}(\theta, x_{1:t}, y_{1:t}) \\
&= \eta(\theta) \prod_{k = 1}^{t} p_{\theta}(x_{k}) \textup{Laplace}\left(y_{k} - s_{k}(x_{k}), \Delta s_{k} / \epsilon \right), \quad t= 1, \ldots, n,
\end{aligned}
\label{eq: sequence of posteriors}
\end{equation}
where we used $\textup{Laplace}(\cdot; b)$ to denote the pdf of $\textup{Laplace}(b)$. A Monte Carlo approximation is necessary for those posterior distributions since they are intractable having no closed form. At time $t$, SMC approximates the posterior distribution in \eqref{eq: sequence of posteriors} with a discrete probability distribution having $N > 1$ particles (points of mass) $\{ (\theta^{(i)}, x_{1:t}^{(i)}); i = 1, \ldots, N\}$ with particle weights $\{ w_{t}^{(i)}; i = 1, \ldots, N\}$ as
\[
p^{\epsilon, N}_{s_{1:t}}(\mathrm{d} (\theta, x_{1:t}) | y_{1:t})  = \sum_{i = 1}^{N} w_{t}^{(i)} \delta_{(\theta^{(i)}, x_{1:t}^{(i)})}(\mathrm{d} (\theta, x_{1:t})).
\]
By marginalising out the $x_{1:t}$ component in the above approximation, we can also obtain the particle approximation of the marginal posterior distribution of $\theta$ given the observations.
\begin{equation} \label{eq: particle approximation of theta}
p^{\epsilon, N}_{s_{1:t}}(\mathrm{d} \theta | y_{1:t})  = \sum_{i = 1}^{N} w_{t}^{(i)} \delta_{\theta^{(i)}}(\mathrm{d} \theta).
\end{equation}
At time-step $t$, the particles and their weights from time $t-1$ are updated after the \emph{resampling}, \emph{rejuvenation}, \emph{propagation}, and \emph{weighting} steps. The propagation and weighting steps are necessary to track the evolving posterior distributions, while the rejuvenation and resampling steps prevent the particle approximation from collapsing to a single point. The update at a single time-step of SMC is detailed in Algorithm \ref{alg: SMC} (time indices of particles are omitted for ease of exposition). The algorithm is an instance of the resample-move algorithm of \citet{gilks_and_berzuini_2001}, specified for the sequence of posteriors in \eqref{eq: sequence of posteriors}. The most common resampling step is multinomial sampling, where the $N$ new particles are sampled independently according to their weights. For the rejuvenation step, one common type of MCMC move consists of (i) an update of $x_{k}$, $k = 1,\ldots, t$, with a Metropolis-Hastings (MH) move with invariant distribution $p_{\theta, s_{k}}(x_{k} | y_{k}) = p_{\theta}(x_{k}) p_{s_{k}}^{\epsilon}(y_{k} | x_{k})$, which is followed by (ii) an update of $\theta$ using an MH move with invariant distribution $p(\theta | x_{1:t}) \propto \eta(\theta) \prod_{k = 1}^{t} p_{\theta}(x)$. One such MCMC move is shown in Algorithm \ref{alg: MCMC} in Appendix \ref{sec: Supplementary algorithms}.  

The computational cost of SMC for processing $n$ observations is $\mathcal{O}(N n^{2})$ in general, since an $\mathcal{O}(t N)$ operation is needed to rejuvenate the particles at time $t$. The cost may be reduced in some cases: The cost for updating $x_{1:t}$ can be reduced by updating a random subset, of a fixed size, of $x_{k}$'s at each time step $t$. The cost for updating $\theta$ may be reduced depending on the model specifics, for example by using a Gibbs move for $\theta$ if the posterior distribution $p(\theta | x_{1:t})$ is tractable.

\begin{algorithm}[t]
\caption{SMC update at time $t$}
\label{alg: SMC}
\KwIn{Particles at time $t-1$, $(\theta_{t-1}^{(1:N)} , x_{1:t-1}^{(1:N)} )$, particle weights $w_{t}^{(i)}$ observation $y_{t}$, function $s_{t}$, DP parameter $\epsilon$}
\KwOut{The particle system at time $t$}
\textbf{Resampling:} Resample particles according to their weights:
\[
(\theta^{(1:N)} , x_{1:t-1}^{(1:N)} )\leftarrow \textup{Resample}((\theta^{(1:N)} , x_{1:t-1}^{(1:N)} ); w_{t-1}^{(1:N)}).
\]
\For{$i = 1, \ldots, N$}{
\textbf{Rejuvenation:} Update $(\theta^{(i)}, x_{1:t-1}^{(i)})$ using an MCMC move that targets $p_{s_{1:t-1}}^{\epsilon}(\theta, x_{1:t-1} | y_{1:t-1})$. \\
\textbf{Propagation:} Sample $x_{t}^{(i)} \sim \mathcal{P}_{\theta^{(i)}}$ and append particle $i$ as $(\theta^{(i)}, x_{1:t}^{(i)}) = (\theta^{(i)}, (x_{1:t-1}^{(i)}, x_{t}^{(i)}))$.
}
\textbf{Weighting:} Calculate $w_{t}^{(i)} \propto \textup{Laplace}\left( y_{t} - s_{t}(x^{(i)}), \Delta s_{t}/\epsilon \right)$ for $i =1, \ldots, N$ s.t.\ $\sum_{i = 1}^{N} w_{t}^{(i)} = 1$.
\end{algorithm}

\section{Adaptive truncation for the transformation} \label{sec: Adaptive truncation for the transformation}
In this section, we focus on $H$ in \eqref{eq: update s general} in Algorithm \ref{alg: Differentially private online learning} and describe a method to determine the function $s_{t}$ adaptively so that the estimation performance of SMC is better over a version where an arbitrary $s_{t}$ is used. We confine to $s_{t}$ that corresponds to truncating $x_{t}$ into an interval $[l_{t}, r_{t}]$,
\[
s_{t}(x) = T_{l_{t}}^{r_{t}}(x) := \min\{ \max \{ x, l_{t} \},  r_{t} \},
\]
so that the sensitivity is $\Delta s_{t} = r_{t} - l_{t}$. We assume $X_{t}$ is univariate; for multivariate $X_{t}$ the truncation approach can be applied to each component. 

How should we choose the truncation points $l_{t}, r_{t}$? Recall the trade-off mentioned earlier: A larger $r_{t} - l_{t}$ renders truncation less likely but leads to a larger noise in $Y_{t}$; whereas a smaller $r_{t} - l_{t}$ renders truncation more likely but leads to a smaller noise in $Y_{t}$. Another critical factor is the location of $l_{t}, r_{t}$ relative to $\theta$. For example, when $\theta$ is a location parameter, an interval $(l_{t}, r_{t})$ around $\theta$ may be preferred. 

\citet{Heikkila_et_al_2017} propose a way to optimise the truncation points for batch estimation; however, their method spends a part of the privacy budget and it is not straightforward to extend their method to online estimation. Following \citet{Alparslan_and_Yildirim_2022}, we use the Fisher information as the amount of information that an observation carries about the population parameter. The Fisher information associated to $Y = T_{l}^{r}(X) + (r-l) V$ when $X \sim \mathcal{P}_{\theta}$ and $V \sim \text{Laplace}(1/\epsilon)$ can be expressed as
\begin{align}
F_{l, r}^{\epsilon}(\theta) &= \mathbb{E} \left[  \nabla_{\theta} \log p^{\epsilon}_{l, r}(Y \vert \theta) \nabla_{\theta} \log p^{\epsilon}_{l, r}(Y\vert \theta)^{T} \right], \label{eq: FIM general - 2}
\end{align}
where $p_{l, r}^{\epsilon}(y \vert \theta)$ is the pdf of the marginal distribution of a single observation $Y_{t} = y$ given $\theta$ and $s_{t} = [l_{t}, r_{t}]$. According to this approach, we set the truncation points $l_{t}, r_{t}$ to those $l, r$ values that jointly maximise $F_{l, r}^{\epsilon}(\theta)$.

$F_{l, r}^{\epsilon}(\theta)$ is smaller for a smaller $\epsilon$, due to more noisy observations. However, it is not obvious how $F_{l, r}^{\epsilon}(\theta)$ behaves with $l, r$. The exact calculation of $F_{l, r}^{\epsilon}(\theta)$ is not possible in general as the truncation of $X$ between $l, r$, if nothing else, introduces an intractability in the calculations. That is why we numerically approximate the Fisher information using Monte Carlo
\begin{equation} \label{eq: Monte Carlo FIM}
F_{l, r}^{\epsilon}(\theta) \approx \frac{1}{M} \sum_{j = 1}^{M} \widetilde{\nabla_{\theta} \log p_{l, r}^{\epsilon}(y^{(j)} | \theta)} \widetilde{\nabla_{\theta} \log p_{l, r}^{\epsilon}(y^{(j)} | \theta)^{T}}, \quad y^{(1)}, \ldots, y^{(M)} \overset{\textup{i.i.d}}{\sim} p_{l, r}^{\epsilon}(y | \theta),
\end{equation}
where each gradient term in the sum is calculated using Algorithm \ref{alg: Monte Carlo calculation of the gradient} in Appendix \ref{sec: Supplementary algorithms}.

When $\theta$ is multidimensional, an overall score function $\text{sc}(\cdot)$ can be used to order the Fisher information matrices. Examples of such a score function are the trace and a weighted sum of the diagonals.

\subsection{Exploration-exploitation for interval selection} \label{sec: Exploration-exploitation for interval selection}

We adjust the interval $[l, r]$ to better estimate $\theta$, yet the adjustment is based on $F(\theta)$, which depends on $\theta$. Therefore, we are adapting the intervals based on something that requires the knowledge of $\theta$ which we want to estimate in the first place. This situation necessitates an \emph{exploration-exploitation} approach. When we have little knowledge about $\theta$, we should let our adaptive algorithm have more freedom to locate the truncation interval; but as we learn $\theta$ by receiving more and more observations, the location of the interval should be chosen with less variety. Our exploration-exploitation approach consists of two steps. Given $\Theta_{t}$, 

\begin{enumerate} [labelwidth={1em},font=\bfseries, align=left, noitemsep, topsep=0pt] 
\item[Step 1] Draw $\vartheta_{t} \sim p^{\epsilon, N}_{s_{1:t}}(\mathrm{d} \theta | y_{1:t})$, the SMC approximation of the posterior distribution at time $t$, i.e.,
\[
\text{choose } \vartheta_{t} = \theta^{(i)} \text{ with probability } w_{t}^{(i)}, \quad i = 1, \ldots, N.
\]
\item[Step 2] Determine the interval for the next observation as 
\begin{equation} \label{eq: exploitation}
l_{t+1}, r_{t+1} = \arg \max_{l, r} \text{sc} (F_{l, r}^{\epsilon}(\vartheta_{t})).
\end{equation}
\end{enumerate}
After determining $[l_{t+1}, r_{t+1}]$, the next data point $X_{t+1}$ is shared as
\begin{equation} \label{eq: truncation and noise-adding}
Y_{t+1} = T_{l_{t+1}}^{r_{t+1}}(X_{t+1}) + (r_{t+1} - l_{t+1}) V_{t+1}, \quad V_{t+1} \sim \textup{Laplace}(1/\epsilon).
\end{equation}
The exploration size is decreased as $t$ increases, that is, as more data are observed. Remarkably, this is automatically handled by Step 1 above, since posterior distribution is spread over a wide region for small $t$ but gets more concentrated as more data are received.

\paragraph{Thompson sampling.} Steps 1 and 2 above can be seen as an instance of \emph{Thompson sampling} in reinforcement learning (see e.g.\ \citet{Russo_et_al_2018}): Using the terminology from reinforcement learning, in our case, the `action' is the choice of the interval $[l_{t}, r_{t}]$, `state' is $Y_{t}$,  the `model parameter' is $\theta$, the `past observations' at time $t$ are the states $Y_{1}, \ldots, Y_{t}$, and the `objective function' is $F_{l, r}(\theta)$. If the maximiser $\arg \max_{l, r} F_{l, r}(\theta)$ is unique for every $\theta$, then Thompson sampling corresponds to first sampling $\vartheta_{t} \sim p^{\epsilon}_{l_{1:t}, r_{1:t}}(\mathrm{d} \theta | Y_{1:t})$ and then setting $l_{t+1}, r_{t+1} = \arg \max F_{l, r}(\vartheta_{t})$, which corresponds to the exploration-exploitation approach described above. The exact implementation of Thompson sampling requires sampling from $p_{l_{1:t}, r_{1:t}}(\mathrm{d} \theta | Y_{1:t})$. As often done in practice, we approximate that step sample from the particle approximation $p_{l_{1:t}, r_{1:t}}^{N}(\mathrm{d} \theta | Y_{1:t})$. %

\subsection{Location and scale parameters and truncation} \label{sec: Location and scale parameters and truncation}
In principle, the maximisation step in \eqref{eq: exploitation} can be applied to any population distribution $\mathcal{P}_{\theta}$ for sensitive data. However, location-scale distribution families deserve particular interest due to their common use and certain desirable properties. It is intuitive to suppose that the best truncation points for a location-scale distribution can be obtained simply by scaling and shifting the best truncation points calculated for some \emph{base} distribution. We show here that this is indeed the case. For a general population distribution $\mathcal{P}_{\theta}$, the maximisation \eqref{eq: exploitation} needs to be performed afresh for each $\vartheta_{t}$. For location-scale families, however, the computationally intensive part of \eqref{eq: exploitation} can be done once for some base distribution and its result can easily be applied for all $\vartheta_{t}$ by scaling and shifting. Below we explain how that is possible.

\begin{definition} \label{defn: location-scale family}
A distribution family $\{ f(\cdot; m, c): (m, c) \in \mathbb{R} \times (0, \infty) \}$ is a location-scale family with a base distribution $g(x)$ if for all $(m, c) \in \mathbb{R} \times (0, \infty)$ we have $f(x; m, c) = \frac{1}{c} g((x - m)/c)$ for all $x \in \mathcal{X}$. In particular, $f(x; 0, 1) = g(x)$.
\end{definition}
Assume that $\mathcal{P}_{\theta}$ is a member of a location-scale family, e.g.\ a normal distribution with $\theta$ being the vector of the mean and the standard deviation. When $\vartheta_{t} = (m, c)$ is sampled in Step 1 above, consider formalising the truncation points as
\begin{equation} \label{eq: truncation points for the location-scale families}
l_{t+1} = a c  + m, \quad  r_{t+1} = b c + m,
\end{equation}
where $a$ and $b$ are the free parameters. Then, the problem in \eqref{eq: exploitation} reduces to choosing the best $a, b$ that maximises $\text{sc}(F_{ac +m, bc +m}^{\epsilon}(m, c))$, where $F_{ac +m, bc +m}^{\epsilon}(m, c)$ is the Fisher information associated to the random variable
\begin{equation} \label{eq: noisy observation}
Y = T_{a c + m}^{b c + m}(X) + c (b - a) V, \quad V \sim \textup{Laplace}(1/\epsilon), \quad X \sim \mathcal{P}_{(m, c)}.
\end{equation}
We show that for location-scale families, a \emph{uniformly best pair} $a, b$ over all possible values $(m, c)$ exists. 
\begin{theorem}  \label{thm: order F}
For any $a, b \in \mathbb{R}$, $\epsilon > 0$ and $(m, c) \in \mathbb{R} \times [0, \infty)$, let $\textup{sc}: \mathbb{R}^{2 \times 2} \mapsto \mathbb{R}$ be a score for information matrices. Then, for all pairs $a, b$ and $a', b'$, either one of the three holds 
\begin{align*}
\textup{sc} (F^{\epsilon}_{a c + m , b c + m}(m, c)) & > \textup{sc} (F^{\epsilon}_{a' c + m, b' c + m}(m, c)  ), \quad \forall (m, c) \in \mathbb{R} \times (0, \infty); \\
\textup{sc} (F^{\epsilon}_{a c + m , b c + m}(m, c)  ) & < \textup{sc} (F^{\epsilon}_{a' c + m, b' c + m}(m, c)  ), \quad \forall (m, c) \in \mathbb{R} \times (0, \infty); \\
\textup{sc} (F^{\epsilon}_{a c + m , b c + m}(m, c)  ) & = \textup{sc} (F^{\epsilon}_{a' c + m, b' c + m}(m, c)  ), \quad \forall (m, c) \in \mathbb{R} \times (0, \infty). 
\end{align*}
\end{theorem}
A proof of Theorem \ref{thm: order F} is given in Appendix \ref{sec: Proof of Theorem}. Theorem \ref{thm: order F} implies that it suffices to find
\begin{equation} \label{eq: best a b for base}
(a^{\ast}, b^{\ast}) = \arg \max_{a, b} \textup{sc}(F_{a, b}^{\epsilon}(0, 1)),
\end{equation}
the best $a, b$ for the base distribution, i.e., for $(m, c) = (0, 1)$. Then, it is guaranteed that those $a^{\ast}, b^{\ast}$ are the best choices for all $(m, c)$ values when the intervals are chosen according to \eqref{eq: truncation points for the location-scale families}. Therefore, maximisation for interval selection needs to be done only once, implying significant computational savings.

\section{Numerical results} \label{sec: Numerical results}
In our experiments\footnote{The code for the experiments can be found at \url{https://github.com/sinanyildirim/SMC_DP_adaTr}}, we take $\mathcal{N}(\mu, \sigma^{2})$ as the population distribution, so that $\theta = (\mu, \sigma)$, and aim to estimate both $\mu$ and $\sigma$. Sensitive data $X_{1}, \ldots, X_{n}$ of length $n = 1000$ are generated from $\mu = 50$ and $\sigma^{2} = 10$. The parameters are taken \emph{a priori} independent with $\mu \sim \mathcal{N}(0, 10^{4})$ and $\sigma^{2} \sim \mathcal{IG}(1, 1)$, where $\mathcal{IG}(\alpha, \beta)$ is the inverse gamma distribution with shape $\alpha$ and scale $\beta$. Algorithm \ref{alg: Differentially private online learning} is implemented for online learning of $\theta$ by combining the SMC method in Section \ref{sec: Sequential Monte Carlo for Bayesian estimation} with the exploration-exploitation strategy described in Section \ref{sec: Adaptive truncation for the transformation} for the choice of the truncation points. 

Note that $\mathcal{N}(\mu, \sigma^{2})$ is a location-scale distribution with $\mu$ and $\sigma$ being the location and scale parameters. Therefore, we apply \eqref{eq: noisy observation} to generate the noisy observations, where $a$, $b$ are the optimal truncation points corresponding to the $\mathcal{N}(0, 1)$.  The Fisher information matrix is a $2 \times 2$ matrix, corresponding to the bivariate parameter $\mu, \sigma$. For the example's sake, we consider that the primary goal is to estimate $\mu$ while $\sigma$ is of secondary importance. Thus, we chose the score function as the first entry of the Fisher information matrix, that is, $\text{sc}(F_{a, b}^{\epsilon}(0, 1)) = F_{a, b}^{\epsilon}(0, 1)[1, 1]$. The maximisation in \eqref{eq: best a b for base} is performed by Monte Carlo estimation of $F_{a, b}^{\epsilon}(0, 1)$ on the $50 \times 50$ grid spanning $[-3, 3] \times [3, 3]$ of $(a, b)$ points. The Monte Carlo estimation is performed as in \eqref{eq: Monte Carlo FIM} with $M = 1000$, where the gradient terms in \eqref{eq: Monte Carlo FIM} are approximated using Algorithm \ref{alg: Monte Carlo calculation of the gradient} with samples of size $10000$. The best $[a, b]$ intervals were numerically found as $[-0.06, 0.06]$, $[-0.12, 0.12]$, $[-0.54, 0.54]$, and $[-0.96, 0.96]$ for $\epsilon = 1, 2, 5, 10$ respectively.

Figure \ref{fig: one experiment} summarises the entire course of one run of SMC with adaptive truncation, which we call ``SMC-adaptive''. For each of $\epsilon = 1, 2, 5, 10$, we repeat this experiment $30$ times independently. 
\begin{figure}[b]
\centerline{
\fbox{\parbox{\textwidth}{
\begin{itemize}
\item First, find $a, b$ that maximises  $\text{sc}(F_{a, b}^{\epsilon}(0, 1))$ the score of the Fisher information matrix of $Y = T_{a}^{b}(X) + (b-a) V$, when $X \sim \mathcal{N}(0, 1)$, $V \sim \text{Laplace}(1/\epsilon)$. 
\item Start with, $l_{1}, r_{1}$.  For $t = 1, \ldots, n$,
\begin{itemize}[noitemsep, topsep=0pt]
\item generate $Y_{t} = T_{l_{t}}^{r_{t}}(X_{t}) + (r_{t} - l_{r}) V_{t}$, where $X_{t} \sim \mathcal{N}(\mu, \sigma^{2})$ and $ V_{t} \sim \text{Laplace}(1/\epsilon)$.
\item Update the particle system of SMC using Algorithm \ref{alg: SMC} with $N = 1000$ particles to construct the SMC approximation of the posterior $p^{\epsilon}_{l_{1:t}, r_{1:t}}(\theta | Y_{1:t})$.
\item sample $\vartheta_{t} = (m, c)$ from the SMC approximation $p^{\epsilon, N}_{l_{1:t}, r_{1:t}}(\theta | y_{1:t})$.
\item determine the new truncation points $l_{t+1} = m + c a$, $r_{t+1} = m + c b$.
\end{itemize}
\end{itemize}
}}
}
\caption{The entire course of one run of the SMC method with adaptive truncation}
\label{fig: one experiment}
\end{figure}

We compared SMC-adaptive to two non-adaptive algorithms. The first one is the same SMC method in Algorithm \ref{alg: SMC}, but with constant truncation points,  $l_{c} = \mu - 10 \sigma$ and $r_{c} =  \mu + 10\sigma$ for all $t$. We call this algorithm ``SMC-non-adaptive''. The second method is an MCMC sampling method that targets the conditional distribution of $\theta$ given the entire batch of the observations at once, $p^{\epsilon}_{l_{c}, r_{c}}(\theta | Y_{1:n})$, where the observations are generated using the same truncation points for all $X_{t}$ as in SMC-non-adaptive as
\[
Y_{t} = T_{l_{c}}^{r_{c}}(X_{t}) + (r_{c} - l_{c}) V_{t}, \quad V_{t} \overset{\textup{i.i.d.}}{\sim} \text{Laplace}(1/\epsilon), \quad t = 1, \ldots, n.
\]
The model for the random variables $\{\theta, X_{1:n}, Y_{1:n}\}$ is a latent variable model for independent observations. That is why, as the MCMC method, we chose the MHAAR algorithm proposed in \cite[Section 3]{Andrieu_et_al_2020}, which is well suited to such latent variable models and also proposed for privacy applications in \citet{Alparslan_and_Yildirim_2022}. The interval $[l_{c}, r_{c}]$, chosen for the non-adaptive methods, represents the situation in many practical applications where there is not much strong \emph{a priori} knowledge available about $\theta$. The comparison with the non-adaptive version of the SMC aims to show the merit of adaptive truncation. Moreover, by comparing with the MCMC method, we aim to show the advantage of adaptation even when online estimation is not required.

Figure \ref{fig: SMC particle distribution} displays the performance of the two SMC methods for a single run and each $\epsilon$. The scatter plots of the particles (after resampling so that they have equal weights) at every 20th time step, as well as the mean estimates, are shown. Further, the truncation points are also shown in the plots for the location parameter $\mu$. Observe the decreasing amount of spread of the particles as $t$. Also, as expected, accuracy increases with $\epsilon$. We also observe the clear benefit of the adaptive truncation method relative to its non-adaptive counterpart when we compare the particle distributions: the particles of the SMC algorithm with adaptive truncation get more concentrated around the true values and do that much more quickly than those of the non-adaptive version. The posterior means, shown with red lines, also demonstrate the advantage of the truncation method.

While Figure \ref{fig: SMC particle distribution} shows results by the SMC methods from a single run, Figure \ref{fig: mean posterior estimates box plots} shows the box plots of the mean posterior estimates of $\theta$, obtained from $30$ independent runs, of all the three methods under comparison, namely SMC-adaptive, SMC-non-adaptive and the MCMC methods. The box plots clearly show that the adaptive truncation approach is beneficial in terms of estimation accuracy as our method beats the other two methods for both parameters and all the tried $\epsilon$ values.

\begin{figure}[h!]
\centerline{
\begin{minipage}{0.25\textwidth}
\begin{tikzpicture}
  \node (img1)  {\includegraphics[scale=0.53]{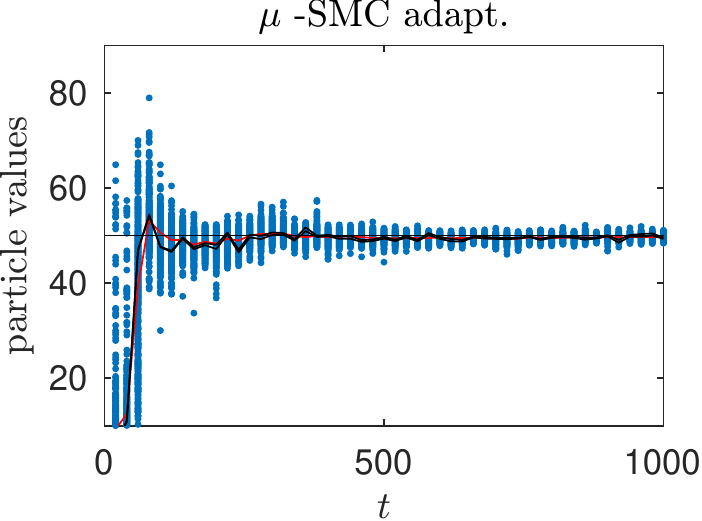}};
  \node[below=of img1, yshift=1cm] (img2)  {\includegraphics[scale=0.5]{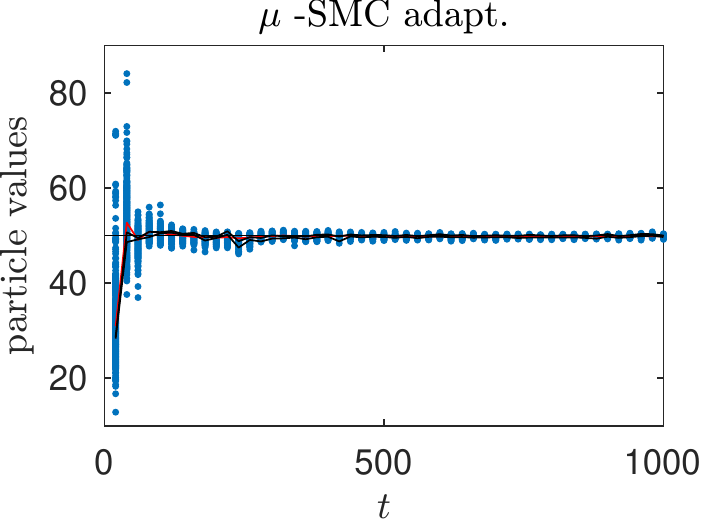}};
  \node[below=of img2, yshift=1cm] (img5)  {\includegraphics[scale=0.5]{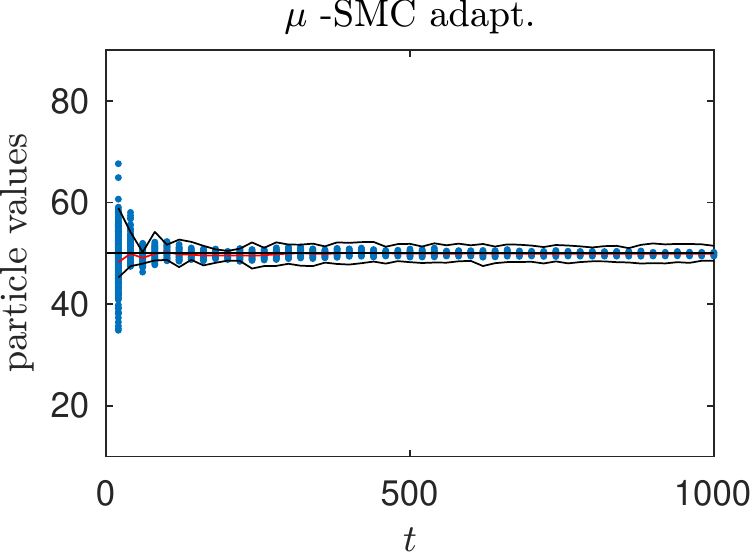}};
  \node[below=of img5, yshift=1cm] (img10)  {\includegraphics[scale=0.5]{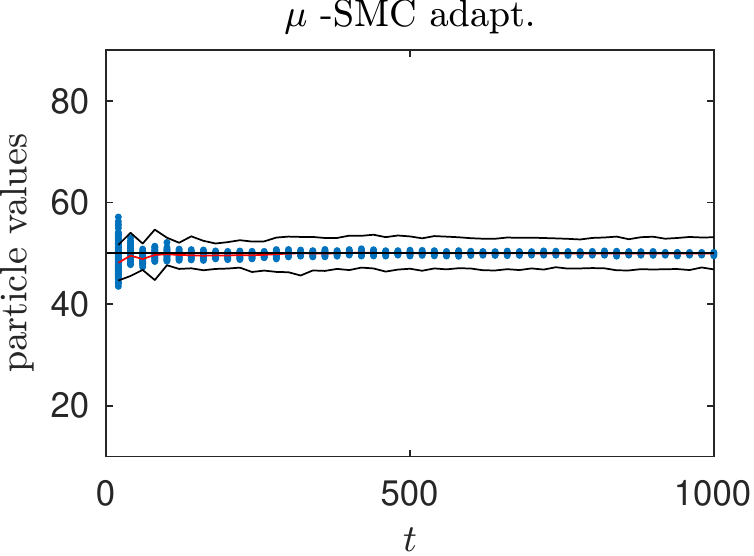}};
   \node[left=of img1, node distance=0cm, rotate=90, anchor=center,yshift=-0.8cm,font=\color{red}] {\small{$\epsilon = 1$}};
  \node[left=of img2, node distance=0cm, rotate=90, anchor=center,yshift=-0.8cm,font=\color{red}] {\small{$\epsilon = 2$}};
  \node[left=of img5, node distance=0cm, rotate=90, anchor=center,yshift=-0.8cm,font=\color{red}] {\small{$\epsilon = 5$}};
  \node[left=of img10, node distance=0cm, rotate=90, anchor=center,yshift=-0.8cm,font=\color{red}] {\small{$\epsilon = 10$}};
 \end{tikzpicture}
\end{minipage}%
\hspace{0.2 cm}
\begin{minipage}{0.25\textwidth}
\begin{tikzpicture}
  \node (img1)  {\includegraphics[scale=0.53]{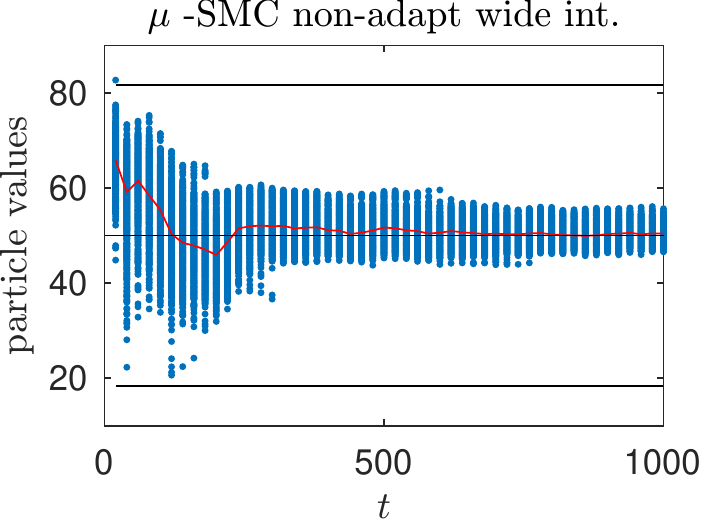}};
  \node[below=of img1, yshift=1cm] (img2)  {\includegraphics[scale=0.5]{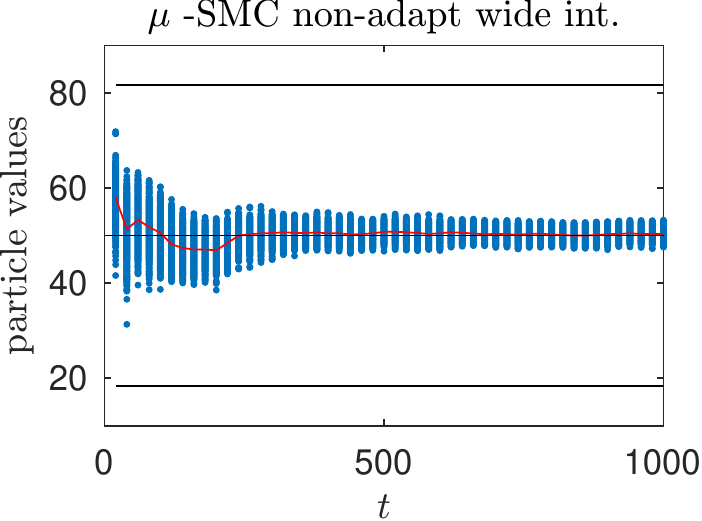}};
  \node[below=of img2, yshift=1cm] (img5)  {\includegraphics[scale=0.5]{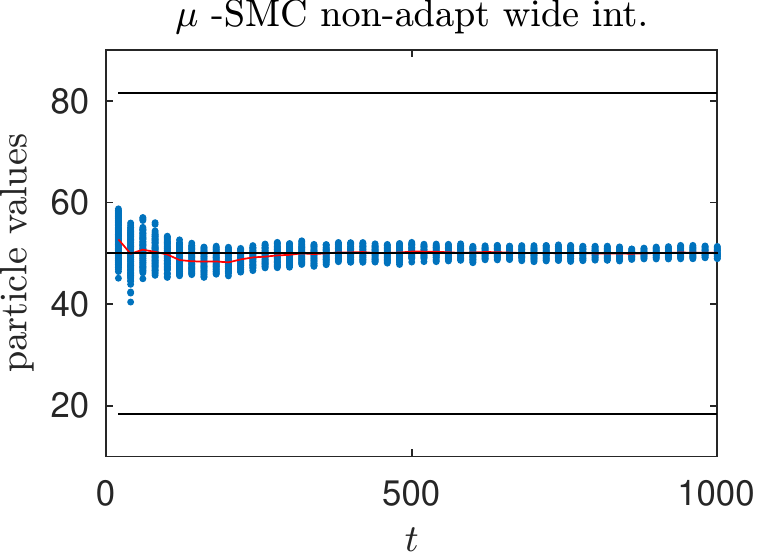}};
  \node[below=of img5, yshift=1cm] (img10)  {\includegraphics[scale=0.5]{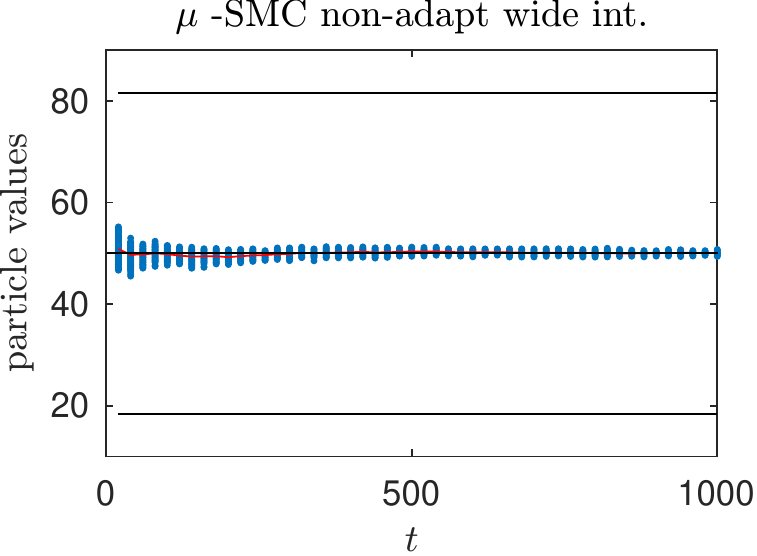}};
 \end{tikzpicture}
\end{minipage}%
\begin{minipage}{0.25\textwidth}
\begin{tikzpicture}
  \node (img1)  {\includegraphics[scale=0.5]{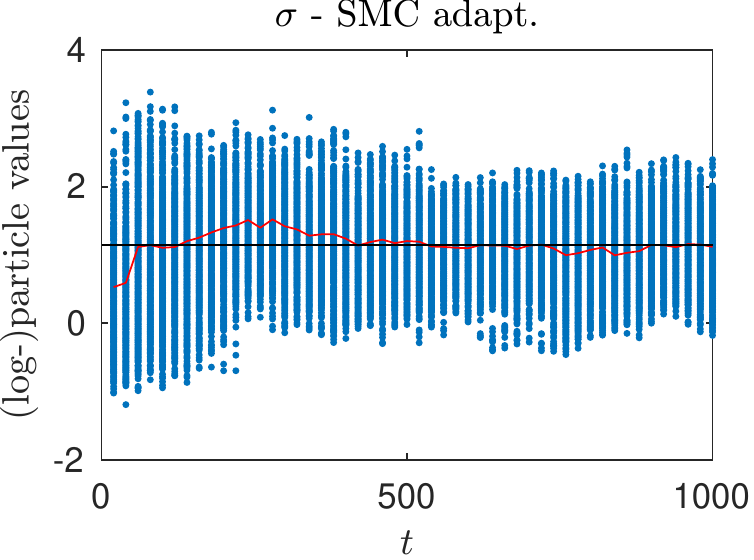}};
  \node[below=of img1, yshift=1cm] (img2)  {\includegraphics[scale=0.5]{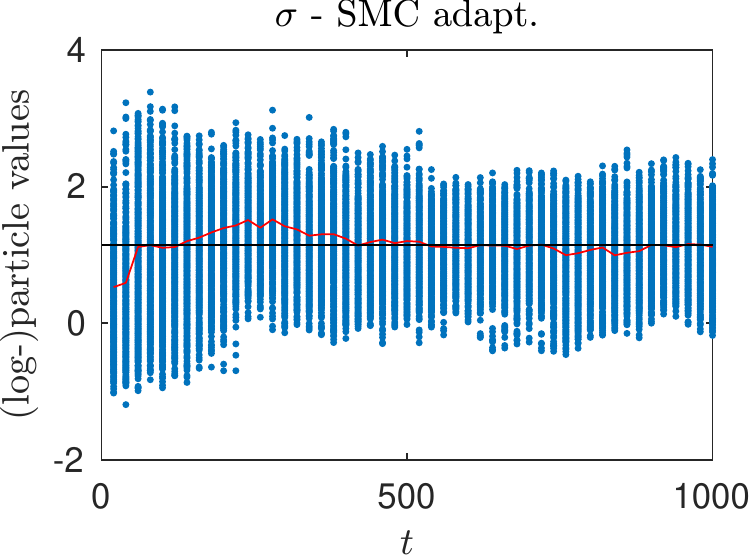}};
  \node[below=of img2, yshift=1cm] (img5)  {\includegraphics[scale=0.5]{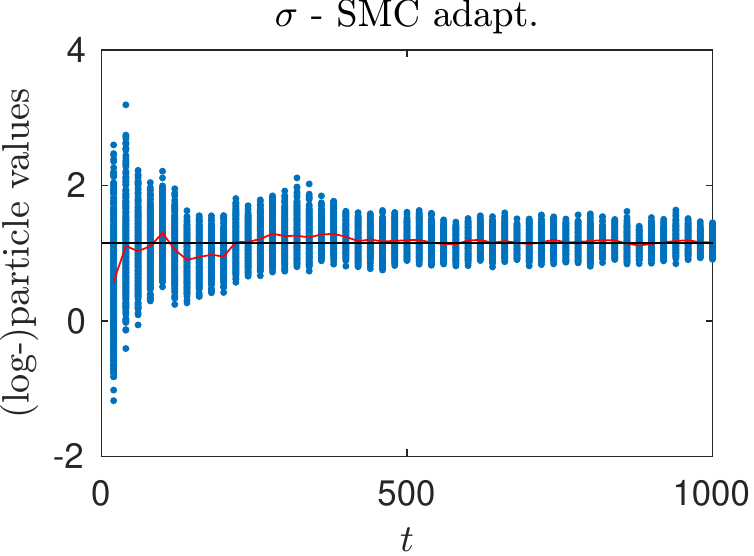}};
  \node[below=of img5, yshift=1cm] (img10)  {\includegraphics[scale=0.5]{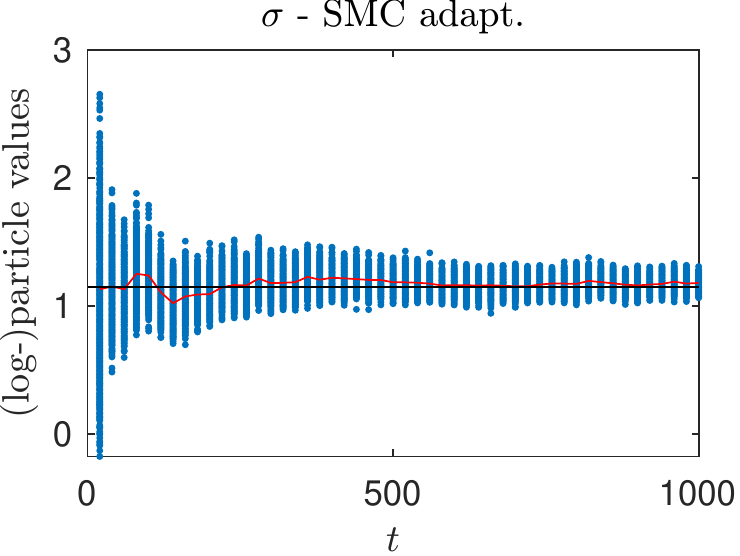}};
 \end{tikzpicture}
\end{minipage}%
\begin{minipage}{0.25\textwidth}
\begin{tikzpicture}
  \node (img1)  {\includegraphics[scale=0.5]{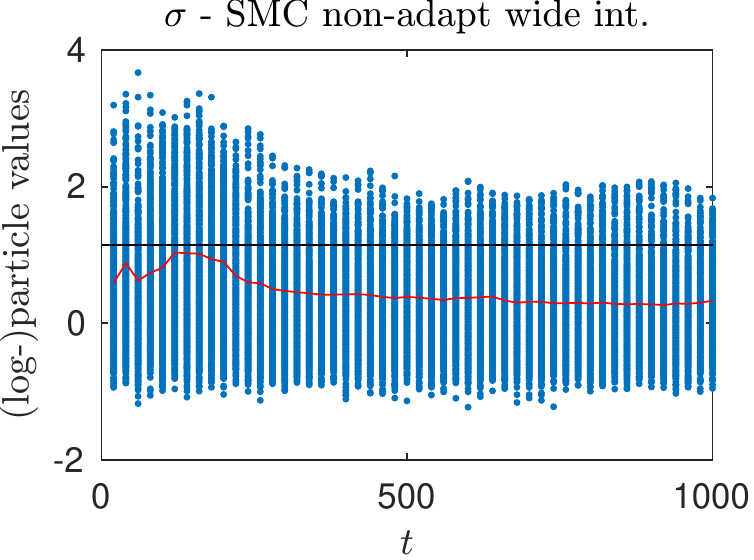}};
  \node[below=of img1, yshift=1cm] (img2)  {\includegraphics[scale=0.5]{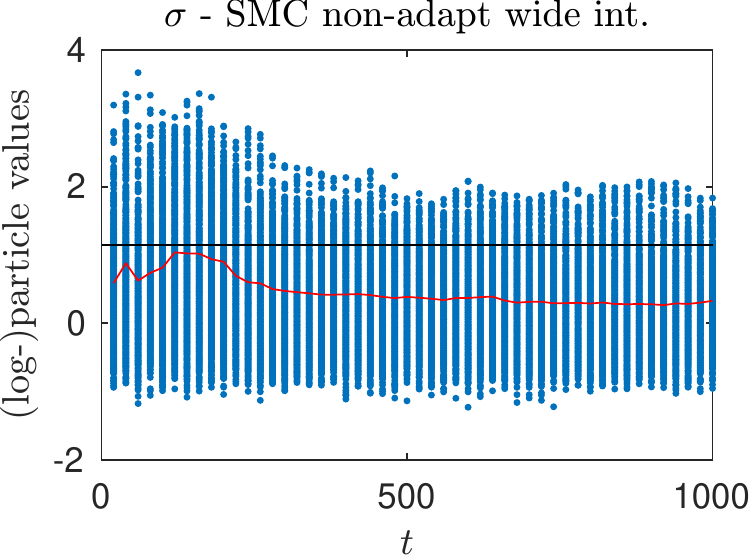}};
  \node[below=of img2, yshift=1cm] (img5)  {\includegraphics[scale=0.5]{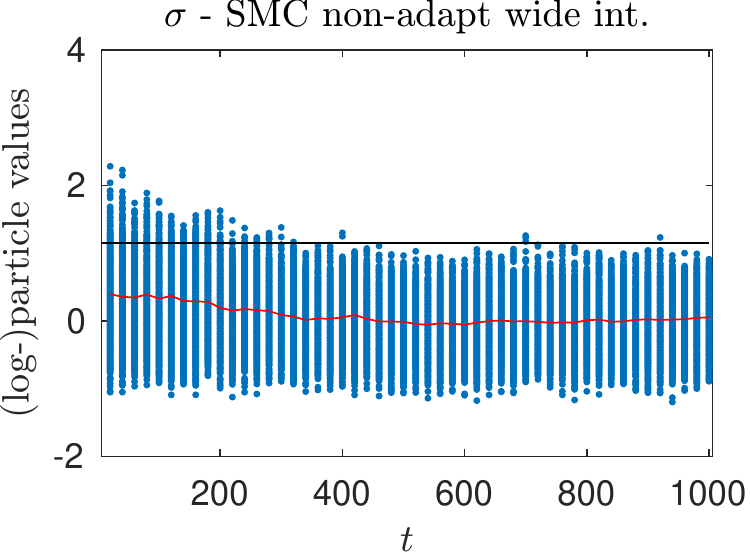}};
  \node[below=of img5, yshift=1cm] (img10)  {\includegraphics[scale=0.5]{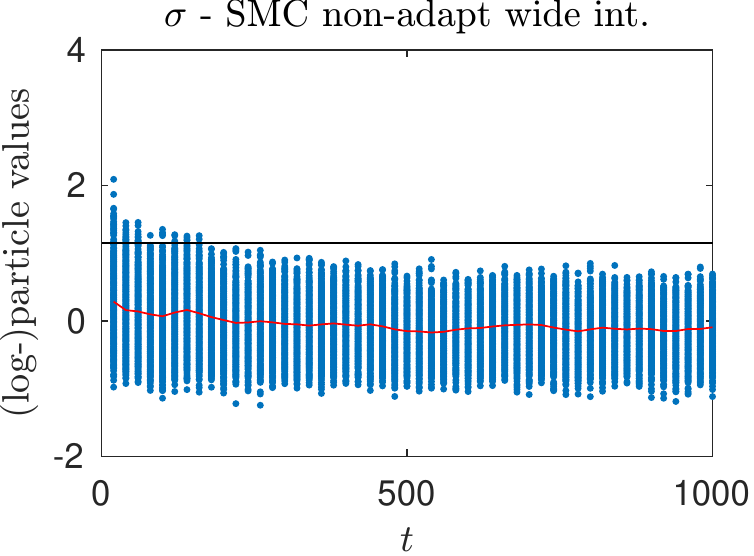}};
 \end{tikzpicture}
\end{minipage}%
}
\caption{The particle distribution of SMC (every 20th time-step shown) (blue points) and the estimate of the posterior means (red line) versus time. Black lines indicate the true values. In plots for $\mu$, truncation points are also shown.}
\label{fig: SMC particle distribution}
\end{figure}
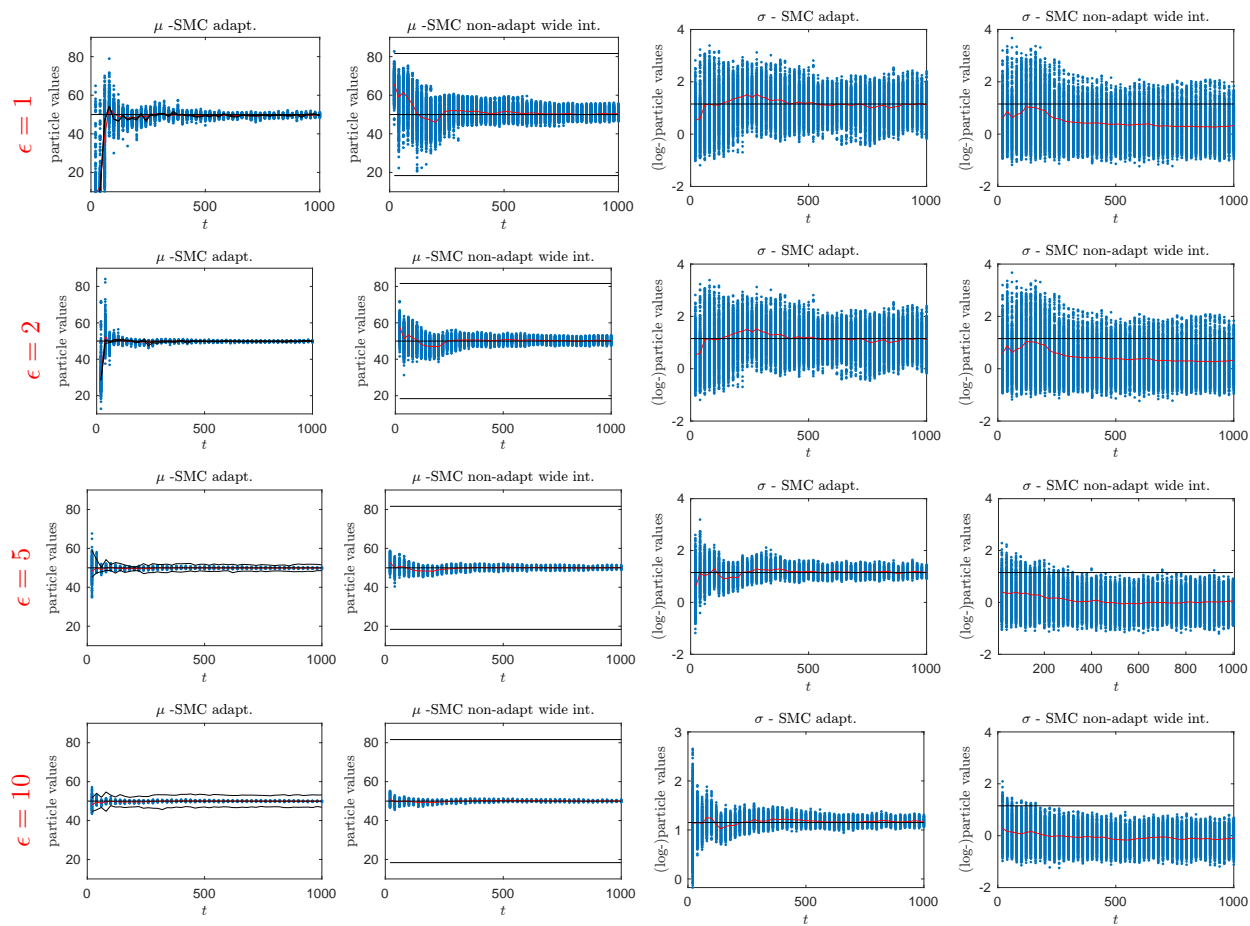

\begin{figure}[h!]
\centerline{
\begin{minipage}{0.25\textwidth}
\begin{tikzpicture}
\node (img1a)  {\includegraphics[scale = 0.47]{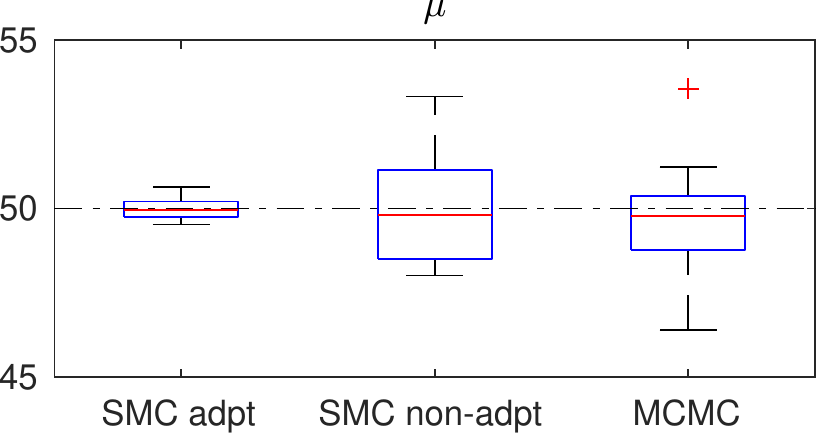}};
\node [below=of img1a, yshift=1cm] (img1b)  {\includegraphics[scale = 0.47]{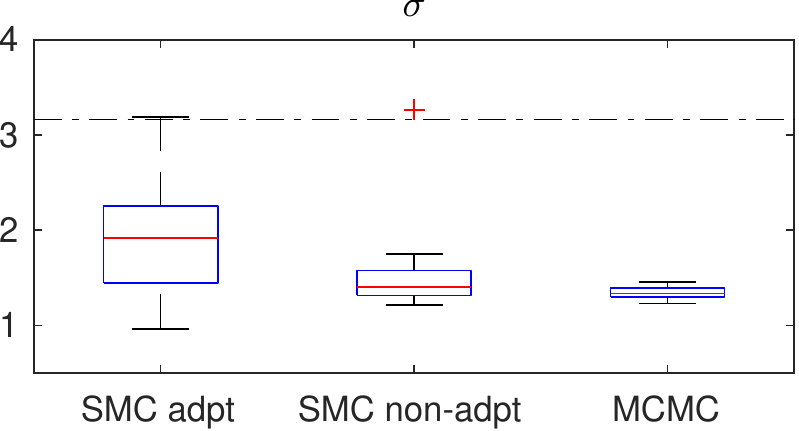}};
\node[above=of img1a, node distance=0cm, anchor=center, yshift=-0.9cm,font=\color{red}] {\footnotesize{$\epsilon = 1$}};
 \end{tikzpicture}
 \end{minipage}
\begin{minipage}{0.25\textwidth}
\begin{tikzpicture}
\node (img1a)  {\includegraphics[scale = 0.47]{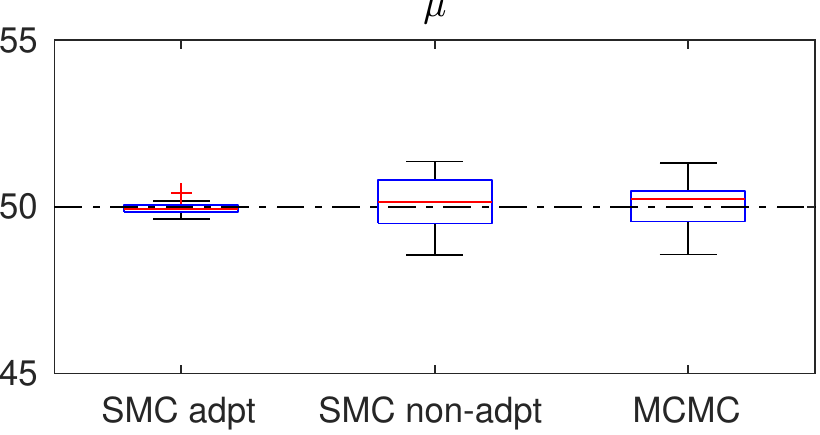}};
\node [below=of img1a, yshift=1cm] (img1b)  {\includegraphics[scale = 0.47]{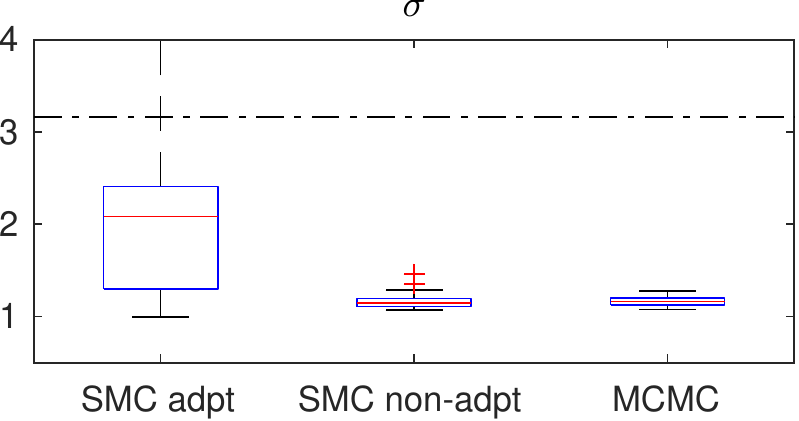}};
\node[above=of img1a, node distance=0cm, anchor=center,yshift=-0.9cm,font=\color{red}] {\footnotesize{$\epsilon = 2$}};
 \end{tikzpicture}
 \end{minipage}
\begin{minipage}{0.25\textwidth}
\begin{tikzpicture}
\node (img1a)  {\includegraphics[scale = 0.47]{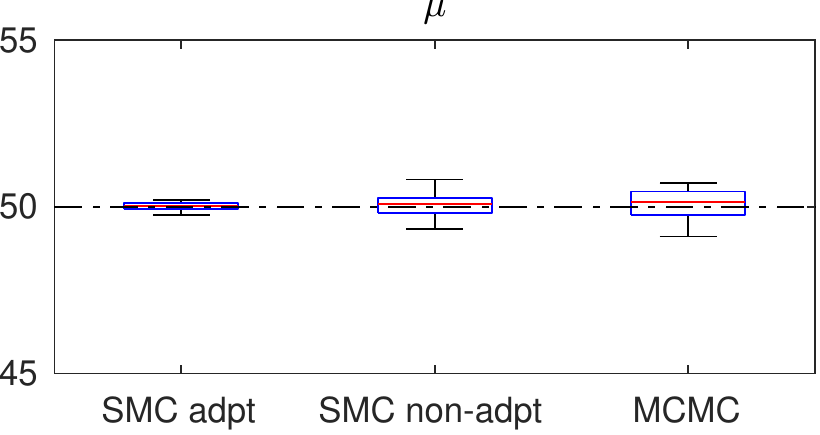}};
\node [below=of img1a, yshift=1cm] (img1b)  {\includegraphics[scale = 0.47]{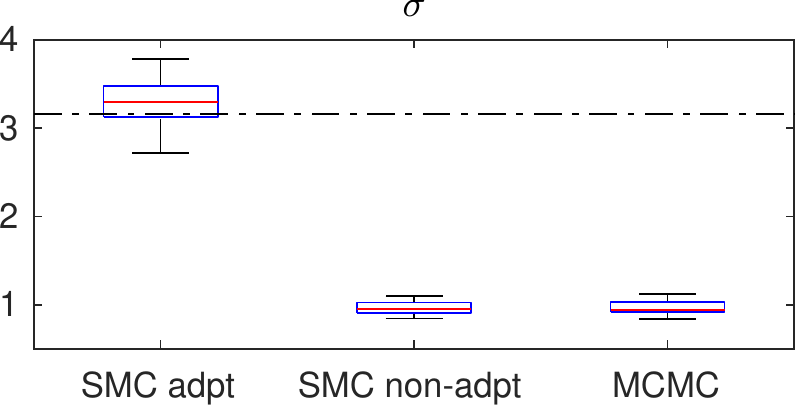}};
\node[above=of img1a, node distance=0cm, anchor=center,yshift=-0.9cm,font=\color{red}] {\footnotesize{$\epsilon = 5$}};
 \end{tikzpicture}
 \end{minipage}
\begin{minipage}{0.25\textwidth}
\begin{tikzpicture}
\node (img1a)  {\includegraphics[scale = 0.47]{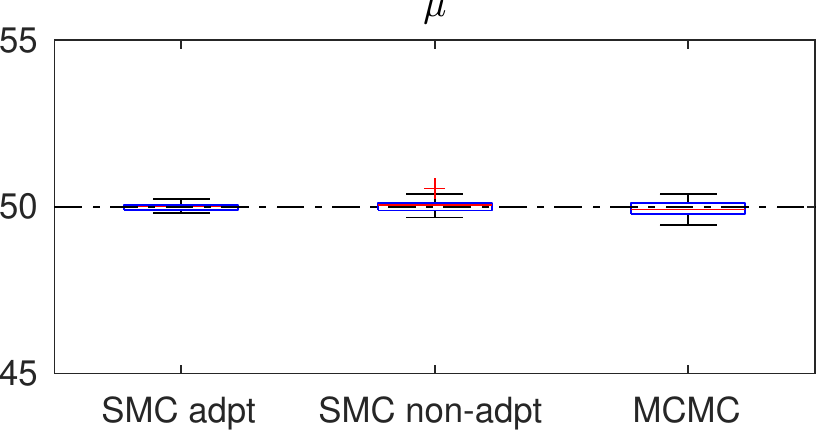}};
\node [below=of img1a, yshift=1cm] (img1b)  {\includegraphics[scale = 0.47]{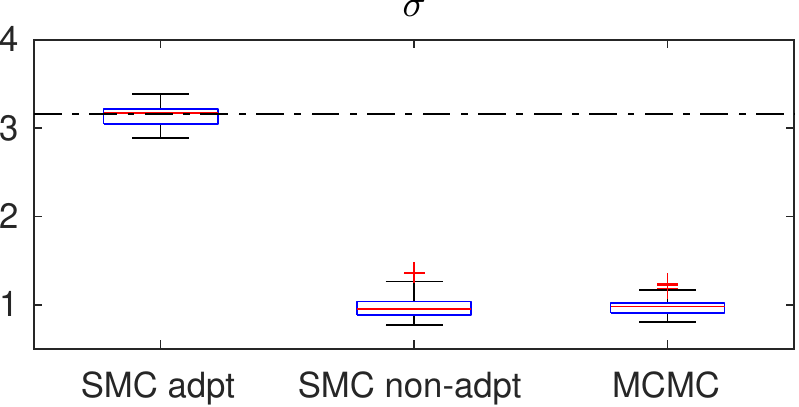}};
\node[above=of img1a, node distance=0cm, anchor=center,yshift=-0.9cm,font=\color{red}] {\footnotesize{$\epsilon = 10$}};
 \end{tikzpicture}
 \end{minipage}
}
\caption{Box-plots of the posterior means, obtained from 30 runs}
\label{fig: mean posterior estimates box plots}
\end{figure}
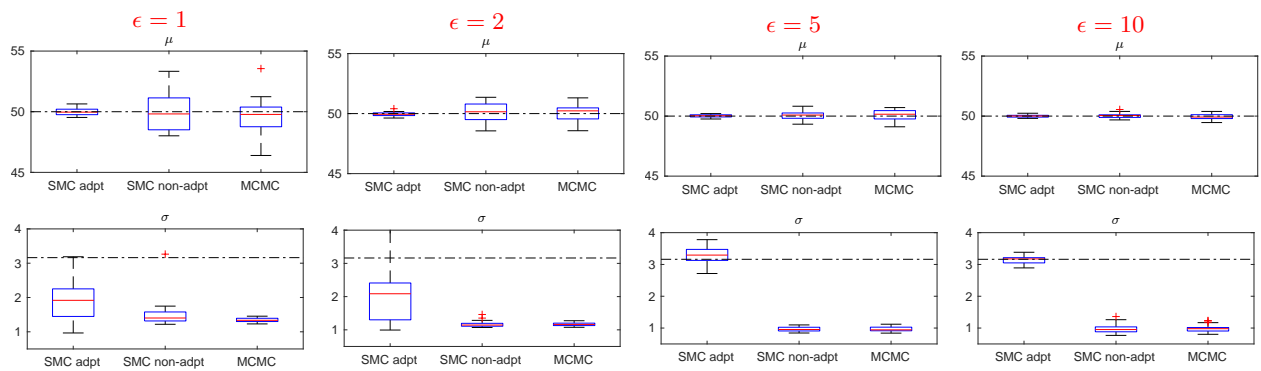

\section{Conclusion} \label{sec: Conclusion}
This paper presents a novel methodology for differentially private online Bayesian estimation with adaptive truncation. The proposed methodology is a working example of the general idea that, as we gain knowledge about the process that generates sensitive data, we can modify our `query' about sensitive data to get more utility while maintaining the same level of privacy. The proposed method demonstrated its merits in the numerical experiments involving the normal distribution, one of the most commonly used distributions for modelling univariate i.i.d.\ data. It would be interesting to see the extension of the work to other distributions, especially multivariate distributions.

Although we considered the Laplace mechanism throughout, the methodology can be modified straightforwardly for other privacy mechanisms, such as the Gaussian mechanism, that provide different senses of privacy. All that changes throughout is the conditional distribution of $Y_{t}$ given $x_{t}, \theta, s_{t}, \epsilon$.

We considered Bayesian inference in this work. Bayesian inference fits ideally into the exploration-exploitation framework by providing a proper sense of uncertainty about $\theta$ via the posterior distribution. However, its computational cost that grows quadratically with data size can be concerning when $n$ is very large. We mentioned some possible ways to reduce the computational load. A viable alternative is to use an online gradient method for online point estimation of $\theta$, where the gradients can be calculated approximately using Monte Carlo as in Algorithm \ref{alg: Monte Carlo calculation of the gradient}. The online gradient method can be advantageous in terms of computational load but it would be more challenging to tune the exploration-exploitation heuristic since a posterior distribution of $\theta$ would not be available.

One limitation of this work is the requirement of considering a parametric distribution family for the population distribution. Estimation of characteristics regarding non-parametric distributions, such the as mean and quantiles, is also a fundamental problem for data streaming applications \citep{Dwork_and_Lei_2009, Alabi_et_al_2022a}. A useful extension of this work would be an online estimation methodology of (characteristics) non-parametric distributions. Non-parametric methods like conformal prediction \citep{Shafer_and_Vovk_2008} can be considered for the adaptive query selection.

\section*{Acknowledgment}
This study was funded by the Scientific and Technological Research Council of Turkey (TÜBİTAK) ARDEB Grant No 120E534. The author is supported by the grant.

\bibliographystyle{apalike}
{\small \bibliography{my_refs.bib}}

\begin{appendix}
\section{Proof of Proposition \ref{prop: DP of Alg 1}} \label{sec: Proof of Proposition 1}

\begin{proof} (Proposition \ref{prop: DP of Alg 1})
Let $R_{n}$ be the set of all revealed outputs of Algorithm \ref{alg: Differentially private online learning}. For any $n \geq 1$, its conditional distribution given $X_{1:n} = x_{1:n}$ at $R_{n} = ( y_{1:n}, s_{1:n}, \theta_{1:n} )$ is given by
\[
P(\mathrm{d} R_{n} | X_{1:n} = x_{1:n}) = \prod_{t = 1}^{n} \textup{Laplace}(y_{t} - s_{t}(x_{t}), \Delta s_{t}/\epsilon ) \mathrm{d} y_{t} \prod_{t = 1}^{n} P(\mathrm{d} s_{t} | \theta_{t}) p(\mathrm{d} \theta_{t} | \theta_{1:t-1}, y_{1:t}, s_{1:t})
\]
The ratio between the conditional distributions with $x_{1:n}$ and $x'_{1:n}$ for any neighbour pair $x_{1:n}, x'_{1:n}$ differing by some $k$'th element is given by
\[
e^{-\epsilon} < \frac{P( \mathrm{d} R_{n} | X_{1:n} = x_{1:n})}{ P(\mathrm{d} R_{n} | X_{1:n} = x'_{1:n})} = \frac{\prod_{t = 1}^{n} \textup{Laplace}(y_{t} - s_{t}(x_{t}), \Delta s_{t}/\epsilon )}{\prod_{t = 1}^{n} \textup{Laplace}(y_{t} - s_{t}(x'_{t}), \Delta s_{t}/\epsilon )} = \frac{\textup{Laplace}(y_{k} - s_{k}(x_{k}), \Delta s_{k}/\epsilon )}{\textup{Laplace}(y_{k} - s_{k}(x'_{k}), \Delta s_{k}/\epsilon )}  < e^{\epsilon}
\]
where the first equality is because the other factors do not depend on $x_{1:n}$ (or $x'_{1:n}$), the second equality is because $x_{1:n}$ and $x'_{1:n}$ differ by the $k$'th element only.
\end{proof}

\section{Proof of Theorem \ref{thm: order F}} \label{sec: Proof of Theorem}
We prove the theorem by first showing that the distribution of $Y$ in \eqref{eq: noisy observation} belongs to a location-scale family. Lemmas \ref{lem: location and scale} and \ref{lem: location and scale of noisy observation} are used for establishing that.  Given a distribution with density $f$ on $\mathbb{R}$ and an interval $[l, r]$, we define $f_{[l, r]}$ to be $f$ truncated to $[l, r]$, that is $f_{[l, r]}(x) \propto f(x) \mathbb{I}\{ x \in [l, r]\}$. Lemma \ref{lem: location and scale} states that the truncated version of a location-scale distribution is also a location-scale distribution.

\begin{lemma} \label{lem: location and scale}
Let $\{ f(x; m, c); m \in \mathbb{R} \}$ be a location-scale family of distributions with location parameter $m$, scale parameter $c$ and base distribution $g$. For any $a, b$; the family of truncated distributions $\{ f_{[c a+m, c b+m]}(x; m, c): (m, c) \in \mathbb{R} \times [0, \infty) \}$ is a location-scale family with location parameter $m$, scale parameter $c$, and base distribution $g_{[a, b]}$.
\end{lemma}
\begin{proof} (Lemma \ref{lem: location and scale})
For all $x \in \mathcal{X}$, $a, b \in \mathbb{R}$, $(m, c) \in \mathbb{R} \times [0, \infty)$, we have
\begin{align*}
f_{[ca+m, cb+m]}(x; m, c) &= \frac{f(x; m, c) \mathbb{I}\{ x \in [ca+m, cb+m]\} }{\int_{ca+m}^{cb+m}  f(u; m, c) du } \end{align*}
Using  $f(x; m, c) = g((x-m)/c)/c$, $\mathbb{I}\{ x \in [ca+m, cb+m]\} = \mathbb{I}\{ (x-m)/c \in [a, b] \}$, and $\int_{ca+m}^{cb+m}  \frac{1}{c}g((u-m)/c) du = \int_{a}^{b}  g(u) du$ by change of variables, we end up with
\begin{align*}
f_{[ca+m, cb+m]}(x; m, c) &= \frac{g((x-m)/c)/c \mathbb{I}\{ (x-m)/c \in [a, b]\}}{\int_{a}^{b}  g(u) du} = \frac{1}{c} g_{[a, b]}((x-m)/c).
\end{align*}
Hence, $f_{[ca+m, cb+m]}(x; m, c)$ is a location-scale distribution with location $m$, scale $c$, base distribution $g_{[a, b]}$.
\end{proof}
Let $f^{\epsilon}_{[ca+m, cb+m]}(x; m, c)$ be the distribution of $Y$ defined in \eqref{eq: noisy observation}. Let $g^{\epsilon}_{a, b}$ be the distribution of $X_{0} + (b-a) V_{0}$ where $X_{0} \sim g_{[a, b]}$, $V_{0} \sim \text{Laplace}(1/\epsilon)$, and $X_{0}$ and $V_{0}$ are independent.
\begin{lemma} \label{lem: location and scale of noisy observation}
 Given $a, b \in \mathbb{R}$ and $\epsilon \in (0, \infty)$, the distribution family $\{ f^{\epsilon}_{[ca + m, cb + m]}(x; m, c): m \in \mathbb{R}, c \in [0, \infty) \}$ is a location-scale family with location $m$, scale $c$ and base distribution $g_{a, b}^{\epsilon}$.
\end{lemma}

\begin{proof} (Lemma \ref{lem: location and scale of noisy observation})
By Lemma \ref{lem: location and scale}, the distribution of $T_{ac +m}^{bc + m}(X)$ is a location-scale distribution with location $m$, scale $c$, and base distribution $g_{a, b}$. For $Y$ in \eqref{eq: noisy observation},  it can be checked that $Y = c Y_{0} + m$ where  $Y_{0} = X_{0} + (b-a) V_{0}$, where $X_{0} = (T_{ac +m}^{bc + m}(X)-m)/c \sim g_{a, b}$, $V_{0} \sim \textup{Laplace}(1/\epsilon)$ and $X_{0}$ and $V_{0}$ are independent. Then, $Y_{0} \sim g^{\epsilon}_{a, b}$, which does not depend on $m$ and $c$. Hence we conclude.
\end{proof}
Finally, we proceed to the proof of Theorem \ref{thm: order F}. 
\begin{proof}(Theorem \ref{thm: order F})
Since the distribution of $Y$ is a location-scale distribution by Lemma \ref{lem: location and scale of noisy observation}, the Fisher information associated to it is given by $F^{\epsilon}_{ac +m, bc+m}(m, c) = \frac{1}{c^{2}}  F^{\epsilon}_{a, b}(0, 1)$, where $F^{\epsilon}_{a, b}(0, 1)$ is the Fisher information matrix associated with the base distribution $g_{a, b}^{\epsilon}(x)$ (for explicit formulae, see, e.g.\ \citet{Jones_and_Noufaily_2015}), and depends on $a$, $b$, and $\epsilon$, but not on $m$ and $c$. Therefore, if $\text{sc}(F_{a, b}^{\epsilon}(0, 1)) > \text{sc}(F_{a', b'}^{\epsilon}(0, 1))$ (resp.\ $<$, $=$), then $\text{sc}(F_{a, b}^{\epsilon}(m, c)) > \text{sc}(F_{a', b'}^{\epsilon}(m, c))$  (resp.\ $<$, $=$) for any other $(m, c) \in \mathbb{R} \times (0, \infty)$.
\end{proof}

%

\section{Supplementary algorithms} \label{sec: Supplementary algorithms}
Algorithm \ref{alg: MCMC} presents an MCMC move for the rejuvenation step of SMC at time $t$. 

\begin{algorithm}[H]
\caption{MCMC for $p^{\epsilon}_{s_{1:t}}(\theta, x_{1:t} | y_{1:t})$ - a single update}
\label{alg: MCMC}
\KwIn{The current sample $(x_{1:t}, \theta)$, proposal distributions $q(x' | x)$ and $q(\theta' | \theta)$, $\epsilon$}
\KwOut{The new sample}
\textbf{MH update for $x_{1:t}$:}\\
\For{$k = 1:t$}{
Sample $x'_{k}  \sim q(x_{k}' |x_{k})$ and return $x'_{k}$ as the new sample w.p.\ 
\[
\min \left\{ 1, \frac{p_{\theta}(x'_{k}) \textup{Laplace}(y_{k} - s_{k}(x'_{k}), \Delta s_{k}/\epsilon) q(x_{k} | x_{k}')}{p_{\theta}(x_{k}) \textup{Laplace}(y_{k} - s_{k}(x_{k}), \Delta s_{k}/\epsilon) q(x'_{k} | x_{k})} \right\};
\]
otherwise return $x_{k}$ as the new sample.
}
\textbf{MH update for $\theta$:} Sample $\theta' \sim q(\theta' | \theta)$ and return $\theta'$ as the new sample w.p.\ 
\[
\min \left\{ 1, \frac{q(\theta | \theta') \eta(\theta') \prod_{k =1}^{t} p_{\theta'}(x_{k})}{q(\theta' | \theta) \eta(\theta) \prod_{k =1}^{t} p_{\theta}(x_{k}) } \right\};
\]
otherwise, return $\theta$ as the new sample.
\end{algorithm}

Algorithm \ref{alg: Monte Carlo calculation of the gradient} approximates Fisher's identity for the score vector, 
\[
\nabla_{\theta} \log p_{\theta, s}^{\epsilon}(y) = \int \nabla \log p_{\theta}(x) p_{\theta, s}^{\epsilon}(x | y) \mathrm{d}x,
\]
using self-normalised importance sampling \citep{geweke_1989} with a sample of size $N$ from $\mathcal{P}_{\theta}$.

\begin{algorithm}[H]
\caption{Monte Carlo calculation of the gradient}
\label{alg: Monte Carlo calculation of the gradient}
\KwIn{Parameter $\theta$, observation $y$, DP parameter $\epsilon$, truncation points $l, r$}
\KwOut{Gradient vector $\widetilde{\nabla_{\theta} \log p_{l, r}^{\epsilon}(y | \theta)}$}
\For{$i = 1, \ldots, N$}{
Sample $x^{(i)} \sim \mathcal{P}_{\theta}$,\\
Calculate $w^{(i)} = \textup{Laplace}(y_{k} - T_{l}^{r}(x^{(i)}), \Delta s_{k}/\epsilon)$.
}
Calculate the (approximate) gradient as 
\[
\widetilde{\nabla_{\theta} \log p_{l, r}^{\epsilon}(y | \theta)} = \frac{\sum_{i = 1}^{N} w^{(i)} \nabla_{\theta} \log p_{\theta}(x^{(i)})}{\sum_{i = 1}^{N} w^{(i)}}.
\]
\end{algorithm}

\end{appendix}

\end{document}